\documentclass[11pt]{article}

\usepackage[margin=1in]{geometry}
\usepackage{amsmath,amsthm,amsfonts,amssymb}
\usepackage[authoryear]{natbib}
\usepackage[hidelinks]{hyperref}
\usepackage{graphicx}
\usepackage{url}
\usepackage{booktabs}
\usepackage{nicefrac}
\usepackage{microtype}
\usepackage{xcolor}
\usepackage{multirow}
\usepackage{multicol}
\usepackage{array}
\usepackage{bm}
\usepackage{float}
\usepackage{enumitem}
\usepackage{subcaption}

\theoremstyle{plain}

\newtheorem{theorem}{Theorem}[section]

\theoremstyle{definition}

\newtheorem{assumption}[theorem]{Assumption}

\newcommand{\change}[1]{#1}
\newcommand{\mycomment}[1]{}

\newcommand{\Y}{\bm Y}
\newcommand{\y}{\bm y}
\newcommand{\x}{\bm x}
\newcommand{\Z}{\bm Z}
\newcommand{\X}{\bm X}
\newcommand{\z}{\bm z}

\newcommand{\V}{\bm V}

\newenvironment{barticle}[1][]{}{}
\newenvironment{binproceedings}[1][]{}{}
\newenvironment{bbook}[1][]{}{}
\def\endbibitem{}
\newcommand{\AND}{ and }
\newcommand{\bauthor}[1]{#1}
\newcommand{\bsnm}[1]{#1}
\newcommand{\bfnm}[1]{#1}
\newcommand{\binits}[1]{}
\newcommand{\byear}[1]{(#1)}
\newcommand{\btitle}[1]{#1}
\newcommand{\bjournal}[1]{\textit{#1}}
\newcommand{\bvolume}[1]{\textbf{#1}}
\newcommand{\bpages}[1]{#1}
\newcommand{\bpublisher}[1]{#1}
\newcommand{\bbooktitle}[1]{\textit{#1}}

\title{Generative Score Inference for Multimodal Data}
\author{
\begin{tabular}{c}
Xinyu Tian and Xiaotong Shen\\
School of Statistics, University of Minnesota\\
\texttt{tianx@umn.edu}, \texttt{xshen@umn.edu}
\end{tabular}
}
\date{}

\begin{document}
\maketitle

\begin{abstract}
Accurate uncertainty quantification is crucial for making reliable decisions in various supervised learning scenarios, particularly when dealing with complex, multimodal data such as images and text. Current approaches often face notable limitations, including rigid assumptions and limited generalizability, constraining their effectiveness across diverse supervised learning tasks. To overcome these limitations, we introduce Generative Score Inference (GSI), a flexible inference framework capable of constructing statistically valid and informative prediction and confidence sets across a wide range of multimodal learning problems. GSI utilizes synthetic samples generated by deep generative models to approximate conditional score distributions, facilitating precise uncertainty quantification without imposing restrictive assumptions about the data or tasks. We empirically validate GSI's capabilities through two representative scenarios: hallucination detection in large language models and uncertainty estimation in image captioning. Our method achieves state-of-the-art performance in hallucination detection and robust predictive uncertainty in image captioning, and its performance is positively influenced by the quality of the underlying generative model. These findings underscore the potential of GSI as a versatile inference framework, significantly enhancing uncertainty quantification and trustworthiness in multimodal learning.
\end{abstract}


\section{Introduction}

Quantifying predictive uncertainty is critical as machine learning models become increasingly complex and are deployed in uncertainty-sensitive domains such as healthcare and finance. While modern deep learning and generative architectures deliver remarkable accuracy, their opaque nature obscures the relationship between inputs and outputs, complicating reliable uncertainty estimation. Compounding this challenge, classical techniques, from asymptotic intervals to bootstrap methods, break down on high-dimensional, unstructured, or multimodal data \citep{dai2022significance,liu2024novel}. As model complexity increases, the risks of overfitting, bias, and lack of reproducibility grow \citep{gibney2022ai}, underscoring the need for advanced uncertainty quantification techniques to address these challenges.

Uncertainty quantification for multimodal data is still in its infancy. By contrast, uncertainty quantification
in tabular regression and classification already benefits from mature, distribution-free guarantees via conformal prediction \citep{vovk2005algorithmic,romano2019conformalized}. Multimodal tasks, however, introduce intricate inter-modal dependencies, heterogeneous noise sources, and vast latent spaces that magnify the hurdles of classical uncertainty
quantification \citep{baltrusaitis2019multimodal}. Although recent specialized approaches, including calibration-based confidence estimation \citep{guo2017calibration, minderer2021revisiting, dai2022significance}, Semantic Entropy \citep{farquhar2024detecting}, and the eigenvalue-based LLM (Large Language Model)-check method \citep{sriramanan2024llm} for hallucination detection, and Conformal Alignment for selection tasks \citep{gui2024conformal}, address specific multimodal issues, a unified, theoretically sound framework applicable broadly to multimodal data remains elusive.

\change{Next-token prediction, the core mechanism of modern large language models, significantly intensifies the challenge of uncertainty quantification. In this framework, the model generates a response sequentially, where each new word or character is predicted based on the initial prompt and the model’s own previously generated sequence. Because the model conditions its output on its own earlier choices rather than on an objective ground-truth continuation, small errors in the initial steps of decoding can propagate and accumulate. This results in responses that remain linguistically fluent and coherent but eventually drift away from the true data distribution, leading to factually unreliable hallucinations. Consequently, traditional uncertainty quantification tools—often designed for simpler, low-dimensional tasks—cannot adequately capture the complexities of these high-dimensional, branching response distributions. Effective quantification must instead account for the uncertainty inherent in the entire generated sequence rather than focusing on individual points or marginal averages.}

In this article, we introduce Generative Score Inference (GSI), a novel framework for uncertainty quantification that constructs prediction sets across diverse tasks, including tabular data, unstructured data, deep learning models, transformers, and selection scenarios. GSI utilizes generative models, such as diffusion models \citep{ho2020denoising} and normalizing flows \citep{rezende2015variational}, to generate high-fidelity synthetic samples that approximate the conditional score distribution---a key measure of predictive performance. This approach enables precise tail probability estimation for uncertainty quantification in complex data settings. By constructing prediction sets or intervals based on these synthetic samples' percentiles, GSI ensures statistical validity with rigorous coverage guarantees, enabling its application to high-dimensional, multimodal, and distributionally complex tasks. This approach becomes a powerful tool for advancing uncertainty quantification in multimodal contexts. 

Empirically, we apply the GSI framework to several representative applications and demonstrate that it outperforms leading methods in the literature across three challenging domains: (1)  
Tabular prediction. Across real-world regression benchmarks, GSI yields narrower intervals with guaranteed coverage, outperforming state-of-the-art conformal methods \citep{shafer2008tutorial,alaa2023conformalized} (Section \ref{tabular}); (2) Hallucination 
detection in large language model (LLM) outputs. The importance of ensuring the trustworthiness of text generated by large language models is underscored by the Nature paper. In Q\&A tasks, GSI surpasses Semantic Entropy \citep{farquhar2024detecting} by more precisely capturing deviations between generated and human-verified text (Section \ref{hallucinations}); (3) Image selection for captioning. On the MS-COCO dataset, GSI exhibits greater statistical power than Conformal Alignment \citep{gui2024conformal} when identifying images that a vision--language model (VLM) can caption reliably (Section~\ref{captioning}).

Our contributions are threefold:

\begin{enumerate}[leftmargin=1.2em]
\item \textbf{A unified framework of uncertainty quantification for multimodal data:} 
We introduce GSI as a general-purpose methodology that constructs prediction sets by estimating conditional score distributions of the response given predictors via generative models. By leveraging Monte Carlo sampling on synthetic score data from a trained generative model, GSI estimates conditional score distributions to construct uncertainty measures. GSI accommodates diverse data modalities (tabular, images, text) and generative model families (diffusion, normalizing flows, autoregressive models), enabling principled and flexible uncertainty quantification across complex data types and prediction tasks.

\item \textbf{Theoretical guarantees:} We establish asymptotic conditional coverage guarantees and provide explicit error bounds for GSI (Section~\ref{sec: theory}). These theoretical results form a rigorous foundation for generative-model-based uncertainty quantification and generative-model-based inference.

\item \textbf{Comprehensive empirical validation and benchmarking:} We evaluate GSI across multiple challenging prediction scenarios, including 
hallucination detection in LLM outputs, image captioning, and tabular regression, and demonstrate clear empirical advantages over existing state-of-the-art methods. In all examined domains, GSI achieves superior statistical power, tighter prediction intervals, and reliable coverage properties, underscoring its effectiveness and broad applicability.
\end{enumerate}

The structure of this article is as follows: Section~\ref{sec: gl} introduces the GSI framework for constructing prediction sets. Section~\ref{sec: theory} presents the statistical guarantees for uncertainty quantification in point prediction. Section~\ref{numerical} compares the performance of GSI with state-of-the-art uncertainty quantification methods in real data applications. Section~\ref{discussion} discusses the methodology and concludes the article. The Supplementary Material includes technical details, with Appendix~A providing proofs, Appendix~B examining diffusion models for conditional generation and analyzing generation error for statistical guarantees, and Appendix~C presenting experimental details.

\section{Generative score inference}
\label{sec: gl}
\subsection{Uncertainty quantification in multimodal tasks}

Uncertainty quantification is essential for multimodal learning systems, where inputs and outputs may span text, images, or their combinations. Among these settings, question answering and image captioning serve as two representative examples for examining uncertainty in large language models and vision--language models. In question-answering tasks, the system receives a user-provided query and returns one or more candidate answers, while image captioning systems take an image as input and generate a natural-language description as output. Despite recent advances, these models frequently exhibit hallucinations---producing confident but factually incorrect responses. As illustrated in Figure~\ref{fig:qa_ic_overview}, the answers generated for the same query may vary substantially across samples, highlighting the inherent uncertainty in their predictions. Entropy-based hallucination detection methods \citep{farquhar2024detecting} build on this observation: when a model is uncertain, it is more likely to hallucinate. However, our experimental results show that hallucinations may still arise even when the model appears highly confident in its response.

The central challenge underlying these applications is determining when we can trust model outputs. Because reference answers or ground-truth captions are typically unavailable at inference time, assessing correctness on a per-example basis is nontrivial. We therefore cast this problem as one of statistical prediction and inference, aiming to quantify uncertainty and provide principled confidence assessments for model-generated outputs.
\begin{figure}[htbp]
    \centering
    \begin{subfigure}{0.9\linewidth}
        \centering
        \includegraphics[width=\linewidth]{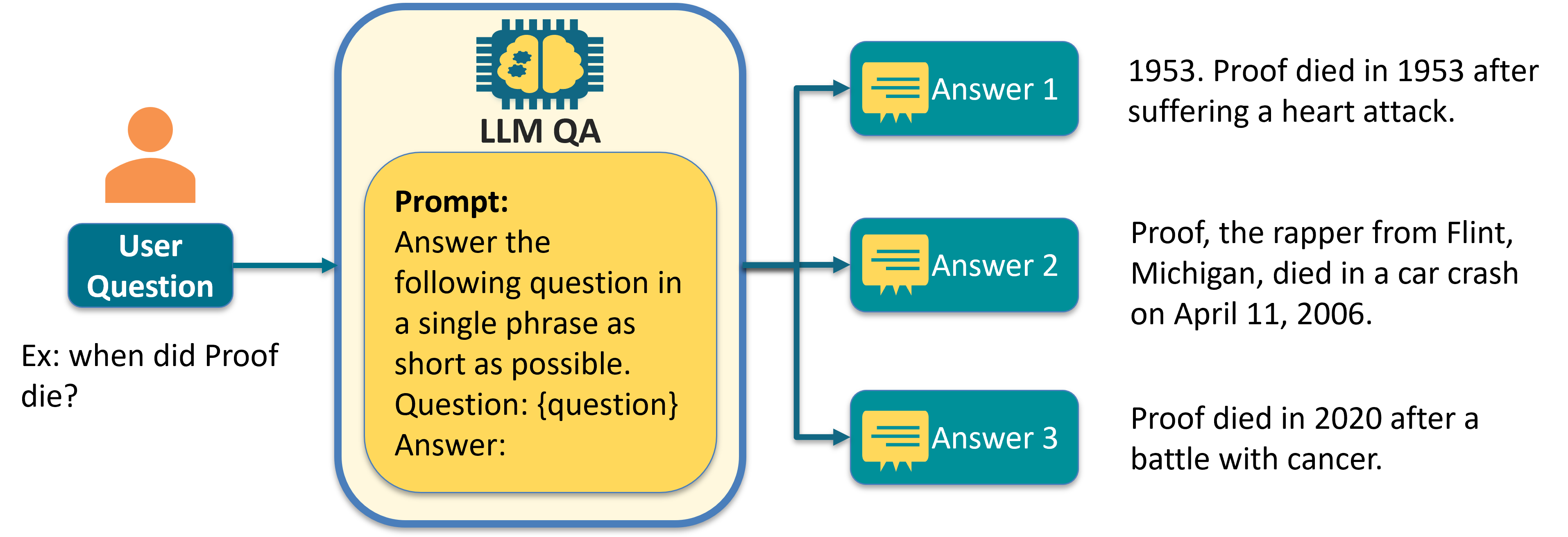}
        \caption{Question--answering module.}
        \label{fig:qa}
    \end{subfigure}

    \vspace{0.8em}

    \begin{subfigure}{0.9\linewidth}
        \centering
        \includegraphics[width=\linewidth]{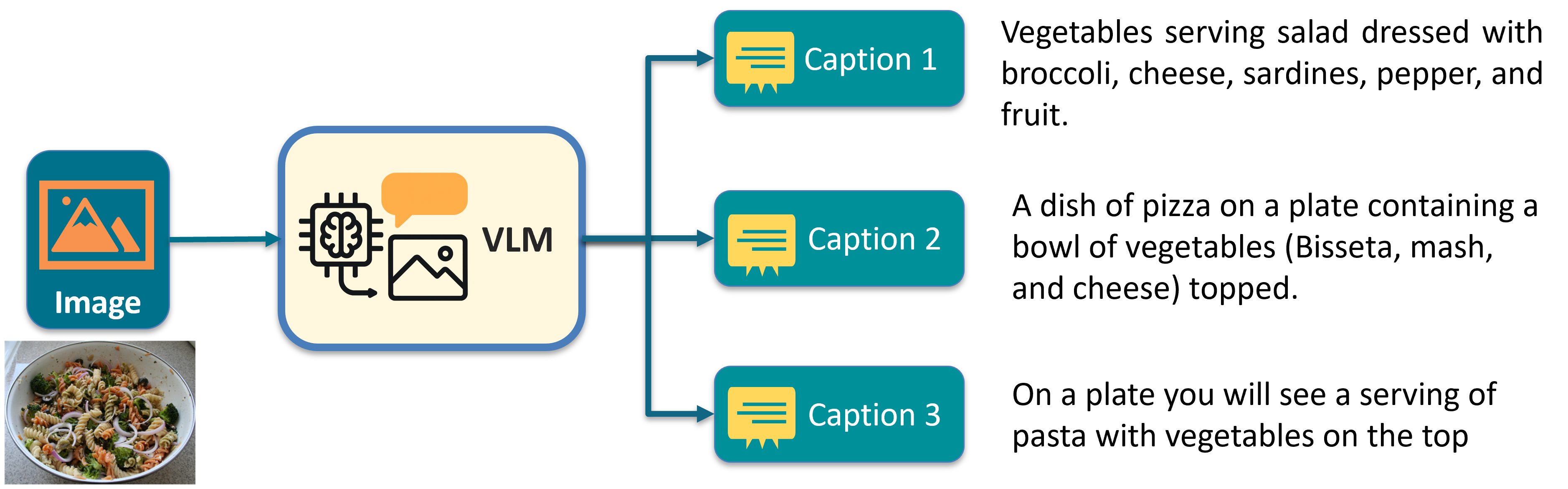}
        \caption{Image captioning module.}
        \label{fig:ic}
    \end{subfigure}

    \caption{Overview of (a) the QA system and (b) the image captioning model.}
    \label{fig:qa_ic_overview}
\end{figure}

\subsection{Prediction sets}
Consider a basic prediction model of the form $\hat{y} = \hat{f}(x)$, where $x$ denotes the predictor and $\hat{f}$ the fitted prediction function. Beyond producing a point estimate, our goal is to construct a prediction set $\mathcal{C}_{\alpha}(x)$ such that the response $y$ lies within this set with high probability. Building on this classical prediction framework, Generative Score Inference provides a model-agnostic approach to uncertainty quantification that is well-suited to multimodal supervised learning.
For a predictor--response pair $(\X,\Y)$,
GSI constructs a prediction set $\mathcal{C}_{\alpha}(\x)$ that guarantees conditional coverage
\[
P_{\y|\x}(\Y \in \mathcal{C}_{\alpha}(\x)) \geq 1 - \alpha, \quad 0 < \alpha<1, 
\]
where $P_{\y|\x}$ is the conditional probability on $\X$, and $1 - \alpha$ is the confidence level. 
The vectors $\X$ and $\Y$ may represent tabular features, images, text, or any combination of heterogeneous data modalities. For example, in the QA setting, 
$\X$ corresponds to the question and 
$\Y$ denotes the answer.

To construct the prediction set, we utilize a score function $s(\y,\hat{\y})$ that quantifies the discrepancy between the observed value $\y$ and the predicted value $\hat{\y}$, leveraging sample splitting. In image captioning, $s(\y, \hat{\y})$ 
could represent metrics such as the aggregate cross-entropy loss or ROUGE-L dissimilarity score (Recall-Oriented Understudy for Gisting Evaluation -- Longest Common Subsequence variant) computed across all word positions. These metrics effectively capture word-by-word differences between predicted and actual captions, spanning multiple sentences \citep{goodfellow2016deep, dai2022coupled}. In tabular regression, $s(\y,\hat{\y})$ could correspond to a measure like the $L_1$-norm. This framework provides a unified methodology for uncertainty quantification across diverse data modalities.

\begin{figure}[h]
\centering
\includegraphics[width=.9\linewidth]{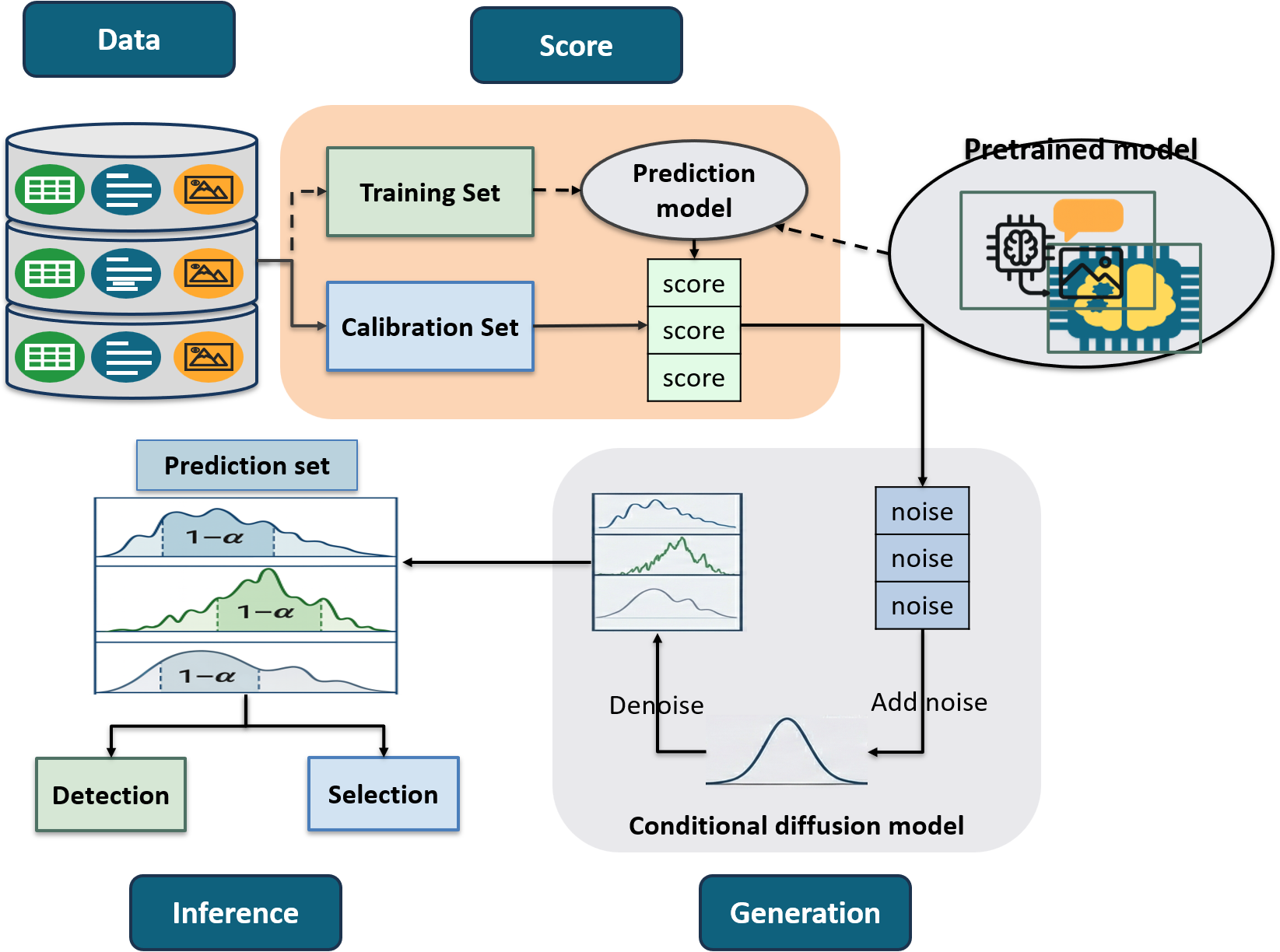}
\caption{Pipeline of GSI method.}
\label{fig_pip}
\end{figure}

Critically, the conditional score distribution encapsulates the uncertainty inherent in the prediction and additional sources of uncertainty, such as those arising from hyperparameter tuning and feature selection---factors often neglected by traditional methods like asymptotic approaches.

Let $q_{1-\alpha}$ represent the $(1-\alpha)$th population quantile of the score distribution. Then, we define the prediction set as:
\[
\mathcal{C}_{\alpha}(\x) = \{ \y : s(\y, \hat f(\x)) \leq q_{1-\alpha} \},
\]
where $\hat{\y}=\hat f(\x)$ denotes the predicted value from the learned prediction function 
$\hat f$ based on the predictor $\x$ and data.

\change{In cases where repeated realizations of $s$ given $\mathbf{x}$ are unavailable, we estimate the score quantile by fitting a conditional generator to the distribution $P_{s\mid \mathbf{x}}$ and then drawing a synthetic sample $\{\tilde{s}_i(\mathbf{x})\}_{i=1}^m$ to compute the empirical $(1-\alpha)$ upper quantile. The GSI framework itself is agnostic to the generator family and can, in principle, be paired with any sufficiently accurate conditional generator. In the present implementation, we use conditional diffusion models because they offer a favorable balance of fidelity, training stability, and flexibility for the heterogeneous score distributions arising in our applications. They are particularly attractive for conditional laws that may be non-Gaussian, heteroscedastic, or multimodal, and they often provide more stable training and better mode coverage than GAN-based methods. Compared with normalizing flows, they avoid invertibility constraints, although flows may offer exact likelihood evaluation and faster sampling. Appendix~C shows empirically that, among the generators considered here, conditional diffusion gives the most stable overall tradeoff between coverage and interval efficiency, while Appendix~B provides the diffusion-specific theory used in our implementation.}

Figure~\ref{fig_pip} illustrates the overall pipeline of GSI.
The GSI method consists of the following steps:

\begin{enumerate}
\item[Step 1] \textbf{Sample splitting and model training.} Randomly split the original data $\mathcal{D} = \{ (\x_i, \y_i) \}_{i=1}^n$ into training and calibration subsets by ratio $1-\rho$ and $\rho$ with a constant 
$0<\rho<1$. Using the training data, we train a prediction model to yield
$\hat{\y} = \hat f(\x)$.

\item[Step 2] \textbf{Score computation.} On the calibration subset, compute the scores $s(\y_i, \hat f(\x_i))$ to obtain the score set $\mathcal{D}_s = \{ (\x_i, s(\y_i, \hat f(\x_i))) \}_{i=1}^{n_s}$, where $n_s=\lfloor n\rho\rfloor$ is the calibration sample size.

\item[Step 3] \textbf{Generative model training.} Train a generative model on $\mathcal{D}_s$ to learn the conditional score distribution $\hat{P}_{s|\x}$. Given a new input $\x_{\text{new}}$, the model estimates $\hat{P}_{s|\x_{\text{new}}}$.

\item[Step 4] \textbf{Prediction set.} Compute the empirical $(1-\alpha)$-upper quantile $q_{1-\alpha,m}$ based on a synthetic sample $\{ \tilde{s}_j \}_{j=1}^m$ from the trained score generative model at $\x = \x_{\text{new}}$.
Define the prediction set:
\[
\mathcal{C}_{\alpha}(\x_{\text{new}}) = \{ \y : s(\y, \hat f(\x_{\text{new}})) \leq q_{1-\alpha,m} \}.
\]
\end{enumerate}

It is worth noting that when a pretrained predictor 
$\hat{f}$
 is available, such as pretrained LLM or VLM, the initial model-training step is unnecessary. In this case, all available data can be devoted to computing the conformity scores and training the generative model for the conditional score distribution.

\textbf{Applications of GSI in Hallucination detection:} 
\change{This framework applies naturally to hallucination detection in LLMs. At deployment time, the reference answer $\Y$ for a new question $\x_{\text{new}}$ is unavailable, so the relevant score distribution must be learned offline from historical or curated question--answer pairs with verified references. We begin by defining a score function that measures semantic dissimilarity between the model output and the expected answer. If a response is declared hallucinated whenever $s(\Y,\hat f(\x_{\text{new}})) > c$, then the problem can be formulated as the following hypothesis test.}
Specifically, consider
\[
H_0:\; s(\Y,\hat f(\x_{\text{new}})) \le c 
\qquad \text{versus} \qquad
H_1:\; s(\Y,\hat f(\x_{\text{new}})) > c .
\]

The prediction set produced by GSI can be employed for this hypothesis test. In practice, we reject $H_0$ if the threshold $c$ is not contained in the prediction set $\mathcal{C}_{\alpha}(\x_{\text{new}})$. Alternatively, one may estimate the tail probability
\[
\hat{P}_{s \mid \x}\{\,s(\Y, \hat f(\x)) > c\,\}
\]
and compare it directly to the significance level $\alpha$.
Hallucinatory outputs yield larger scores, leading to higher tail probabilities and thus rejection of \(H_0\).

More comprehensive demonstrations of this approach for detection and selection problems are provided in Sections~\ref{hallucinations} and~\ref{captioning}.

\textbf{Connection with conformal prediction:} 
Both conformal prediction and Generative Score Inference provide formal guarantees for prediction sets, but they differ in flexibility and precision. Classical conformal prediction constructs a prediction set by taking the \((1-\alpha)\)-quantile of nonconformity scores on a calibration set, thereby ensuring marginal coverage under the exchangeability assumption \citep{shafer2008tutorial}. However, this guarantee is unconditional over the calibration data and does not translate to conditional coverage. Although recent extensions of conformal prediction aim to achieve approximate conditional coverage, they are typically confined to tabular prediction tasks \citep{jung2023batch,alaa2023conformalized}. 
By contrast, GSI leverages generative models to estimate the full conditional distribution of the score function. This capability makes GSI particularly well-suited for high-dimensional or unstructured data, where direct quantile estimation may be less reliable. While GSI can produce tighter and better-calibrated prediction sets, it incurs higher computational costs due to sampling or density estimation in complex spaces. When resources permit, GSI's direct modeling of the conditional score distribution more accurately captures the variability of predicted outcomes in complex tasks.

\section{Statistical guarantee}
\label{sec: theory}

\subsection{\change{Conditional coverage}}

Recall that $s = s(\y, \hat{f}(\x_{\text{new}}))$, where $\hat{f}(\x_{\text{new}})$ represents the predicted value.
Let $\mathrm{TV}(P_{s|\x_{\text{new}}}, \hat{P}_{s|\x_{\text{new}}})$ denote the total variation distance between the true distribution $P_{s|\x_{\text{new}}}$ and the synthetic distribution $\hat{P}_{s|\x_{\text{new}}}$, generated 
by a generative model.

\begin{assumption}[Generation error]
\label{A_generator}
\change{For any tolerance error $\tau>0$, the probability of} 
$
\mathrm{TV}(P_{s|\x_{\text{new}}}, \hat{P}_{s|\x_{\text{new}}}) \leq \tau
$
is greater than $1 - \beta(\tau, n_s)$, where $\beta(\tau, n_s) \to 0$ as $n \to \infty$, and $n_s = \rho n$. 
\end{assumption}

This condition describes the impact of the accuracy of conditional generation. A strong generator can accurately capture the score distribution even when the predictors are less predictive of the outcome. Conversely, highly predictive features can simplify the score distribution. In both cases, we demonstrate that GSI retains desired inferential properties when the score distribution is well-estimated through conditional generation. \change{In deep generative settings, the error term $\tau$ can reflect finite-sample estimation error, imperfect mode coverage, approximation error in the reverse diffusion process, optimization error in overparameterized neural networks, or mismatch between the chosen architecture and the true conditional score law. These practical sources of error motivate the additional validation step used in our experiments to reduce residual finite-sample bias.}

The function $\beta(\tau, n_s)$ can be derived for diffusion and normalizing flow models, as demonstrated by \citep{tian2024enhancing}. The Appendix provides a detailed derivation of $\beta(\tau, n_s)$ for diffusion models, employing generation accuracy theory.

\begin{theorem}[GSI's conditional coverage]
\label{thm1}
Under Assumption \ref{A_generator}, for any Monte Carlo tolerance error 
$\varepsilon > 0$ and generation tolerance error $\tau > 0$, the conditional probability coverage error of 
prediction sets given $\x_{\text{new}}$ is bounded by:
\[
|P_{\y|\x_{\text{new}}}(\Y_{\text{new}} \in \mathcal{C}_{\alpha}(\x_{\text{new}})) - (1-\alpha)| \leq \varepsilon + \tau,
\]
with probability at least $1 - 2\exp(-2m\varepsilon^2) - \beta(\tau, n_s)$. Consequently,
\[
\lim_{\tau, \varepsilon \to 0} \lim_{n,m \to \infty} P_{\y|\x_{\text{new}}}(\Y_{\text{new}} \in \mathcal{C}_{\alpha}(\x_{\text{new}})) = 1 - \alpha,
\]
with probability tending to one.
\end{theorem}

Theorem \ref{thm1} provides a statistical guarantee of the validity of GSI's prediction set $\mathcal{C}_{\alpha}(\x_{\text{new}})$, ensuring that it contains the future outcome $\Y_{\text{new}}$ with the desired level of confidence. It also suggests that the GSI's confidence level is primarily determined by the generation error $\tau$, as the Monte Carlo error $\varepsilon$ is user-controlled and can be made arbitrarily small.
It is worth emphasizing that our results establish a conditional coverage guarantee, which is stronger than the marginal coverage guarantees provided by classical conformal inference \citep{angelopoulos2021gentle}. Similar conditional guarantees in a probably-approximately-correct (PAC) style appear in prior work \citep{jung2023batch, alaa2023conformalized}, but those results are confined to the quantile-regression setting.

This theoretical bound also indicates the necessity of a bias-correction step when the generation error is large. To address this, we introduce an additional calibration set (referred to as the validation set) to realign the empirical coverage or significance level with the nominal target. Practically, this is achieved by selecting the largest $\alpha$ that attains the desired nominal level on the validation set. We will describe this procedure in detail within the specific experimental examples.

\section{Experiments}
\label{numerical} 
To demonstrate the application of the GSI method to uncertainty quantification in multimodal tasks, we begin by evaluating it on a classical tabular prediction task to demonstrate its superior conditional coverage performance relative to conventional approaches (Section~\ref{tabular}). We then extend our analysis to hallucination detection in a QA setting using the LLaMA-3.1-8B-Instruct model \citep{llama3modelcard} (Section~\ref{hallucinations}), followed by an application to image captioning (Section~\ref{captioning}). Additional experimental details are provided in Appendix~C.

\subsection{Tabular prediction}
\label{tabular}

This subsection assesses GSI on tabular regression and compares it with several conformal-prediction methods
across five public benchmarks.

\noindent\textbf{Baselines.} 
We benchmark against the split Conformal Prediction (CP) \citep{shafer2008tutorial} and Conformal Unconditional Quantile Regression (CUQR) \citep{alaa2023conformalized}. The CUQR method is noted for strong conditional-coverage guarantees \citep{alaa2023conformalized}, which have the same focus with our method.

\noindent\textbf{Datasets.} 
We evaluate GSI on five standard tabular benchmarks from the UCI repository and prior conformal-prediction work: MEPS-20, Bio, Kin8nm, Naval, and Blog \citep{feldman2021improving,romano2019conformalized,chung2021beyond,alaa2023conformalized}. 
MEPS-20 comprises 17,541 samples with 139 features describing U.S.\ medical expenditures. Bio contains 45,730 samples of physicochemical protein descriptors with 9 features. Kin8nm includes 8,192 observations and 8 kinematic features of a robotic arm. Naval consists of 11,934 samples and 17 features characterizing yacht propulsion. Finally, Blog comprises 52,937 posts with 280 features capturing weblog popularity.

\begin{table*}[htbp]
\caption{Performance on five benchmarks at the nominal $90\%$ level, summarized by marginal coverage ($C_{\text{marg}}$), average interval length ($L_{\text{avg}}$), and worst-case subgroup coverage ($C_{G}$). Higher coverage and shorter intervals indicate better performance. Baseline results for CP and CUQR are reproduced from \citep{alaa2023conformalized}. The boldface marks the best result in each metric.}
\label{tab:sup_table}
\centering
\setlength{\tabcolsep}{3pt} 
\begin{tabular}{l|ccc|ccc|ccc|ccc|ccc}
\toprule
& \multicolumn{3}{c|}{\textbf{MEPS--20}}
& \multicolumn{3}{c|}{\textbf{Bio}}
& \multicolumn{3}{c|}{\textbf{Kin8nm}}
& \multicolumn{3}{c|}{\textbf{Naval}}
& \multicolumn{3}{c}{\textbf{Blog}} \\
\midrule
& $C_{\text{marg}}$ & $L_{\text{avg}}$ & $C_{G}$
& $C_{\text{marg}}$ & $L_{\text{avg}}$ & $C_{G}$
& $C_{\text{marg}}$ & $L_{\text{avg}}$ & $C_{G}$
& $C_{\text{marg}}$ & $L_{\text{avg}}$ & $C_{G}$
& $C_{\text{marg}}$ & $L_{\text{avg}}$ & $C_{G}$ \\
\midrule
CP & 0.90 & 1.24 & 0.15 & \change{\textbf{0.90}} & 2.42 & 0.85 & 0.90 & 2.17 & 0.83 & \change{\textbf{0.89}} & 1.31 & 0.78 & 0.89 & 1.89 & 0.57 \\
CUQR & 0.89 & 1.21 & 0.76 & \change{\textbf{0.90}} & 2.40 & \textbf{0.88} & 0.89 & 2.19 & \textbf{0.85} & 0.86 & 1.26 & \textbf{0.85} & 0.87 & 1.82 & 0.67 \\
GSI & \change{\textbf{0.91}} & \textbf{0.98} & \textbf{0.80} & 0.89 & \textbf{2.19} & \textbf{0.88} & \change{\textbf{0.90}} & \textbf{1.82} & \textbf{0.85} & \change{\textbf{0.89}} & \textbf{0.96} & \textbf{0.85} & \change{\textbf{0.90}} & \textbf{1.66} & \textbf{0.80} \\
\bottomrule
\end{tabular}
\end{table*}

\noindent\textbf{Experimental protocol.} 
Each dataset is split $85{:}15$ into training and test portions, after \cite{alaa2023conformalized}. 
The training portion is further divided $50{:}50$ into training and calibration splits. 
A gradient-boosting regressor is fitted on the training data for all methods. \change{In this tabular setting, $\X$ denotes the covariate vector, $\Y$ the scalar response, and the score function is the absolute residual $s(\Y,\hat f(\X)) = |\Y-\hat f(\X)|$. Consequently, the GSI prediction set is an interval centered at the fitted regression prediction. For each test point, GSI uses $m=1000$ Monte Carlo draws from the learned conditional score distribution.}

\noindent\textbf{Evaluation metrics.} All methods are evaluated in terms of
\begin{itemize}[leftmargin=1.2em]
\item \emph{Marginal coverage} \(C_{\text{marg}}\): the overall coverage rate, which should be at least the nominal level \(1 - \alpha\).
\item \emph{Average interval length} \(L_{\text{avg}}\): shorter intervals indicate greater efficiency of the prediction sets.
\item \emph{Worst-case subgroup coverage} \citep{alaa2023conformalized}: 
$
C_{G} \;=\; \min_{g} C_{\text{marg}}(\widehat{\mathcal{S}}_{g})$, 
where \(C_{\text{marg}}(\widehat{\mathcal{S}}_{g})\) is the coverage within subgroup \(g\). The \(G=10\) subgroups are defined by \(k\)-means clustering on the training data. A higher \(C_{G}\) indicates more reliable conditional coverage across all subgroups.
\end{itemize}

\noindent\textbf{CUQR calibration.} 
CUQR introduces an additional validation step that searches for the largest nominal level~\(\alpha\) whose interval attains the desired coverage, thereby compensating for score-distribution bias \citep{alaa2023conformalized}. 
To ensure a fair comparison, we apply the same bias-correction rule when calibrating GSI.

\noindent\textbf{Empirical comparison.} The target coverage level is set as \(1-\alpha = 0.90\).
Table~\ref{tab:sup_table} reveals that only GSI and CUQR show better subgroup coverage than the original conformal prediction method. Crucially, GSI delivers shorter prediction intervals than CUQR while matching---or surpassing---its conditional coverage, showcasing the benefits of generative modelling. GSI shrinks intervals by $10$--$20\%$ on average.
GSI also displays markedly lower variance in both interval length and coverage across data domains, highlighting its robustness and strong out-of-sample generalization.

Compared with CP and CUQR, GSI leverages a generative model that learns the entire conditional distribution of $s$ given $\X$, thereby capturing heteroscedastic and highly skewed residuals. By sampling from this conditional distribution, GSI constructs highest-density regions whose width adapts to local predictive uncertainty. This yields shorter intervals in low-noise regions without sacrificing coverage in high-noise subgroups, thereby increasing $C_{G}$.

\change{Appendix~C further reports an additional comparison of four score generators on the same five tabular datasets: conditional diffusion, conditional VAE \citep{sohn2015learning}, conditional GAN \citep{mirza2014conditional}, and conditional flow \citep{winkler2019learning}. Among these methods, diffusion and conditional flow exhibit similar marginal coverage overall, but diffusion shows more stable subgroup coverage while maintaining competitive interval lengths. The conditional VAE performs competitively on some datasets, although it generally produces longer intervals, whereas the conditional GAN tends to undercover despite its shorter intervals. These findings support the use of conditional diffusion as a robust default generator in the current implementation of GSI.}

In summary, GSI's generative mechanism produces data-adaptive prediction sets whose behavior adjusts naturally to the underlying uncertainty structure. This robustness across heterogeneous tabular domains underscores the benefits of combining confidence calibration with modern generative modeling.

\subsection{Detection of hallucinations in LLM outputs}
\label{hallucinations}

This subsection tackles the pressing challenge of hallucination detection---scenarios in which large language models (LLMs) generate fluent yet factually incorrect or nonsensical answers. 
Ensuring the trustworthiness of text generated by LLMs remains a critical challenge.

In this experiment, we evaluate three competitive detectors on responses produced by the \texttt{LLaMA-3.1-8B-Instruct} model for the WikiQA corpus \citep{yang-etal-2015-wikiqa}, which provides 1,473 question-answer pairs with human-verified ground-truth labels. Following standard practice, we partition the corpus into 1,040 calibration instances, 140 validation instances, and 293 held-out test instances for evaluation. Since this experiment utilizes a pre-trained LLM, there is no need for an additional training set to fit the prediction model $f$. Instead, the calibration instances are directly employed to construct the score set $\mathcal{D}_s$, i.e., $n_s = 1040$.

\change{This setup clarifies the role of reference answers in practice. Although a future test question does not come with a verified answer at decision time, the method requires reference answers offline for calibration, validation, and evaluation. When such references are scarce, GSI can still be deployed using curated benchmark sets, periodic human annotation, or other proxy-supervision pipelines, but its reliability then depends directly on the quality and representativeness of that auxiliary supervision. We therefore view the present experiment as most relevant to domains with audited historical responses rather than to fully reference-free online deployment.}

\noindent\textbf{Hypothesis-testing formulation.} For a given question $\mathbf X$ with reference answer $\Y$, 
let $\hat f(\X)$ be the model output. Define a semantic fidelity score $s(\Y,\hat f(\X))$ and fix an acceptability threshold $c$; hallucination detection is then cast as
\[
H_0: s(\Y,\hat f(\X)) \le c \quad\text{versus}\quad H_a: s(\Y,\hat f(\X)) > c,
\]
with $c$ representing a predefined threshold of acceptable semantic dissimilarity.
To test, we compute semantic dissimilarity over the calibration and validation samples,
\begin{equation*}
s(\Y,\hat f(\X)) = 1 - \mathrm{cosine\_sim}(\phi(\Y), \phi(\hat f(\X))).
\end{equation*}
\change{Here, $\phi(\cdot)$ denotes the text embedding and $\mathrm{cosine\_sim}$ the cosine similarity. The embedding is obtained from OpenAI's \texttt{text-embedding-3-small} model. Under this definition, $s(\Y,\hat f(\X)) \in [0,1]$, and large dissimilarity indicates poor alignment with the reference and thus strong evidence against $H_0$.}

\noindent\textbf{Labeling policy.} 
Consistent with prior work \citep{gui2024conformal}, we assign a label to the data for evaluation by setting $c=0.7$, which means any response whose cosine similarity with the reference falls below 0.3 is flagged as a hallucination, denoted by $z=1$. 

\noindent\textbf{Calibration strategy.} Since repeated outputs per test question are unavailable, we calibrate on the validation split to control the Type I error rate---the proportion of non-hallucinatory cases (\(z=0\)) incorrectly flagged as hallucinations. We tune the decision threshold to ensure the empirical Type I error does not exceed the nominal level.

The test is conducted in the following two steps:

\begin{enumerate}[leftmargin=1.2em]

\item \emph{Estimation of Type I error on the validation set:} 
For a grid of candidate significance levels \(\{\alpha_j\}\), construct the corresponding prediction sets, reject \(H_{0}\) whenever the threshold \(c\) falls outside the set (i.e., when \(P_{s\mid x}(s>c)>\alpha_j\)), and record the empirical Type I error. Then select the largest level \(\alpha_{*}\) whose observed error does not exceed the nominal level \(\alpha\).

\item \emph{Output results on the test set:} 
For each test instance, compute the prediction set at level \(\alpha_{*}\), and reject \(H_{0}\)---thus flagging the instance as a hallucination---whenever \(c\) lies outside this set.

\end{enumerate}

\noindent\textbf{Methods under comparison.} We benchmark our Generative Score Inference (GSI) against two top performers, Semantic Entropy (SE) \citep{farquhar2024detecting} and Conformal Alignment (CA) \citep{gui2024conformal} in the same setting. 
\begin{itemize}[leftmargin=1.2em]
\item \change{\textbf{GSI} learns the conditional distribution of semantic dissimilarities using a diffusion model and constructs answer-specific prediction sets with the synthetic sample size fixed at $m=1000$.}
\item \textbf{SE} quantifies uncertainty by measuring the dispersion of multiple generated candidates. For each question, we generate 10 answers, compute the SE, and then convert these raw entropy scores into probability $P(z=1)$ using a logistic regression model trained on the calibration data.

\item \textbf{CA} fits an XGBoost classifier \citep{chen2016xgboost} to learn the alignment score, which is the probability that the testing object belongs to the $H_0$ set.
\end{itemize}

All three methods share the same null and alternative hypotheses (\(H_0\) and \(H_a\)), but differ in how they estimate the decision probability---either \(P(z=1)\) or \(P(s>c)\)---and in the rejection criterion (i.e., whether this probability exceeds the significance level \(\alpha\)). To ensure a fair comparison, we apply the identical calibration procedure described above when generating test results for each method.

\noindent\textbf{Empirical findings.} 
We evaluate statistical power while maintaining exact Type I error control at each nominal $\alpha$. 
As depicted in Figure \ref{fig_qa}, both GSI and CA approach a power of $1$ as $\alpha$ increases, whereas SE saturates at a markedly 
lower level. Notably, GSI consistently outperforms CA and SE for practically relevant $\alpha$-values, demonstrating 
superior power without compromising the validity of Type I control.

\begin{figure}[h]
\centering
\includegraphics[width=.95\linewidth]{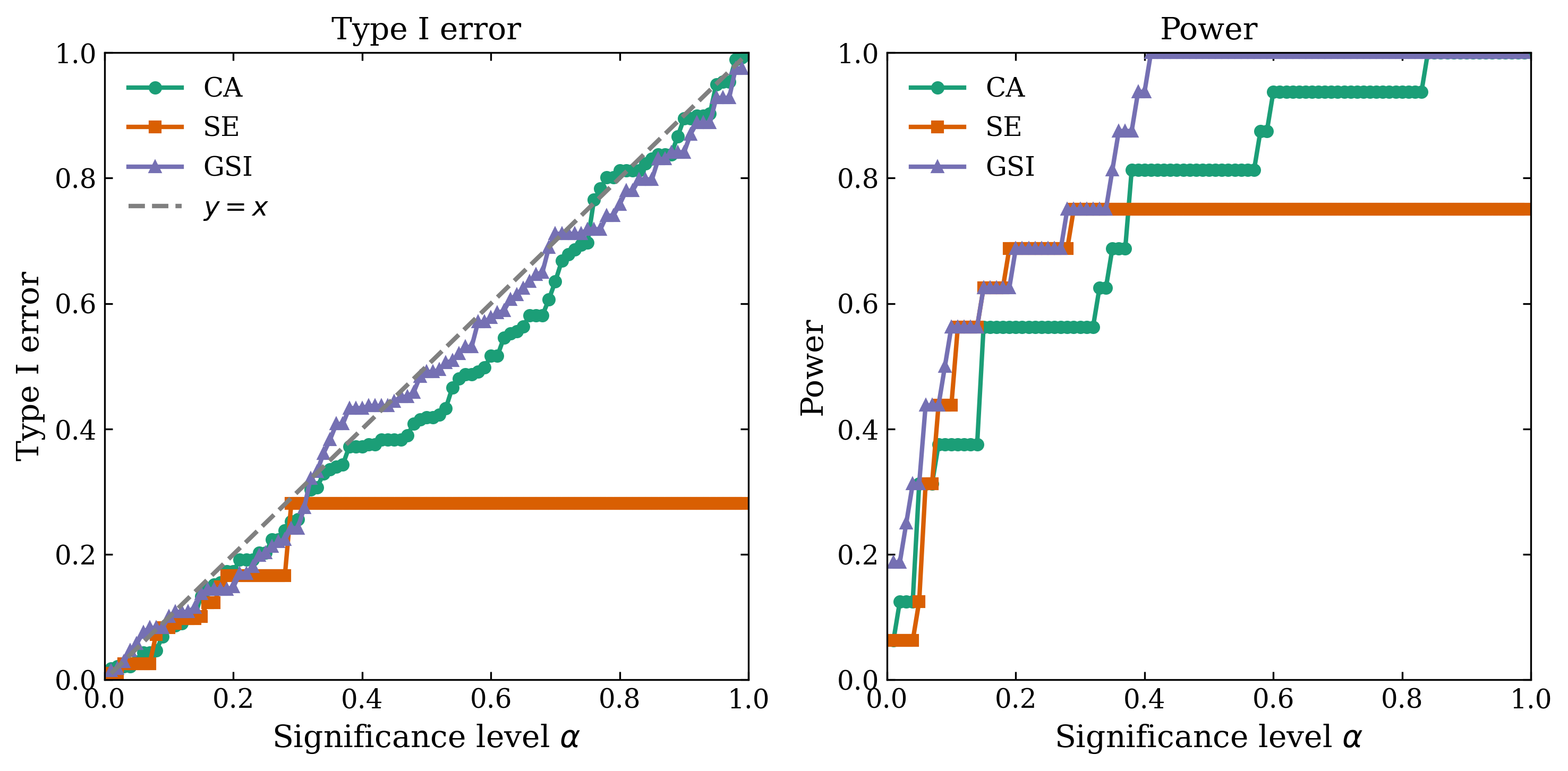}
\caption{Type I error and power comparison for SE \citep{farquhar2024detecting}, CA \citep{gui2024conformal}, and GSI across varying $\alpha \in (0,1)$. Smaller $\alpha$ indicates higher confidence $(1 - \alpha)$. The diagonal line $y = x$ represents ideal Type~I error control.}
\label{fig_qa}
\end{figure}

Table \ref{tab:qa_scores} provides qualitative insight into the comparative behavior of the three detectors. In Example 1, these detectors generate several mutually inconsistent results. SE assigns a low-entropy (overconfident) score that is indistinguishable from that of the true-answer case, while CA yields an equivocal probability that falls below its rejection threshold. Both miss the hallucination. By contrast, GSI explicitly measures divergence from the gold answer and correctly flags the output as hallucinatory. In Examples 2 and 3, all three methods agree, declining to flag a hallucination in the former and detecting one in the latter, mirroring human judgments and underscoring that GSI's superior sensitivity does not come at the expense of specificity.

\noindent\textbf{Why does GSI excel?} 
The superiority of GSI over SE and CA can be attributable to two principled advantages:
\begin{enumerate}[leftmargin=1.2em]
\item \emph{Reference-aware scoring.} 
GSI and CA compute their statistics with explicit respect to the human-verified reference answer, whereas SE relies solely on internal diversity across generated candidates. 
By anchoring every prediction to ground truth, GSI (and, to a lesser degree, CA) can expose hallucinations that appear coherent in a self-referential sense but diverge from factual content.
\item \emph{Answer-level uncertainty quantification.} 
GSI models the full conditional distribution of semantic dissimilarities with a diffusion process, yielding answer-specific credibility sets that sensitively capture nuanced semantic mismatches. 
In contrast, SE compresses variability into a single entropy value while CA employs an XGBoost point estimator, which may overlook subtle but critical deviations.
\end{enumerate}
Together, these properties position GSI as a robust safeguard against hallucinations, offering high sensitivity and precise control of false alarms in next-generation LLM deployments.

\begin{table*}[ht]
\renewcommand{\arraystretch}{1.15}
\caption{Illustrative examples of generated answers and their testing results for given questions in the Q\&A task. The '*' indicates a rejection, i.e. the probability $P(z=1)$ given in the method larger than the threshold $\alpha_*$, which is derived from the validation set under the significance level $\alpha=0.05$. \label{tab-qa}}
\centering
\scriptsize
\begin{tabular}{%
>{\raggedright\arraybackslash}p{1.5 cm} 
>{\raggedright\arraybackslash}p{3 cm} 
>{\raggedright\arraybackslash}p{3 cm} 
l l l l 
}
\toprule
\bfseries Question & \bfseries Reference Answer & \bfseries Generated answer 
& \bfseries Score & \bfseries SE & \bfseries CA & \bfseries GSI \\
\midrule
\multirow{3}{=}{Who is on blink\,182 album cover?}
& \multirow{3}{=}{The cover artwork for Enema of the State features porn star Janine Lindemulder famously clad in a nurse uniform.}
& The band members are on the Blink-182 album cover.
& 0.883 & 0.028 & 0.007 & 0.117* \\
& & The Blink-182 album cover features the band members as children. 
& & ($\leq$0.143) & ($\leq$0.102) & (>0.078) \\
& & The album cover features the faces of the band members.
& & & & \\
\addlinespace[5pt]

\multirow{3}{=}{What is a popular people meter?}
& \multirow{3}{=}{A people meter is an audience measurement tool used to measure the viewing habits of TV and cable audiences.}
& A popular people meter is a device used to track television viewership and audience preferences.
& 0.418 & 0.028 & 0.0291 & 0.045 \\
& & A popular people meter is an electronic device used to measure TV ratings and viewer preferences.
& & ($\leq$0.143) & ($\leq$0.102) & ($\leq$ 0.078) \\
& & A popular people meter is a device used in television ratings to measure the viewing habits.
& & & & \\
\addlinespace[5pt]

\multirow{3}{=}{When did Proof die?}
& \multirow{3}{=}{DeShaun Dupree Holton (October 2, 1973 -- April 11, 2006), better known by his stage name Proof, was an American rapper and actor from Detroit, Michigan.}
& Proof, the rapper from Flint, Michigan, died in a car crash on April 11, 2006.
& 0.820 & 0.159 * & 0.104 * & 0.132 * \\
& & Proof died in 2020 after a battle with cancer.
& & ($>$0.143) & ($>$0.102) & ($>$0.078) \\
& & Proof died in 1953 after suffering a heart attack.
& & & & \\
\midrule
$\alpha_*$
& & 
& & 0.143 & 0.102 & 0.078 \\
\bottomrule
\end{tabular}
\label{tab:qa_scores}
\end{table*}

\subsection{Selection in image captioning}
\label{captioning}

This subsection targets the companion problem of image--caption selection: identifying images that a captioning system can describe accurately while rigorously controlling false discoveries. 
We work with the first 10,000 images of the COCO 2014 validation dataset (\url{https://cocodataset.org}) and adopt the same three-way partition as in Section \ref{hallucinations}: 8,000 images for calibration, 1,000 for validation, and 1,000 for testing. Captions are generated with the pretrained \texttt{BLIP} model \citep{li2022blip}, eliminating the need to train a captioning model from scratch (i.e., $n_s=8000$).

\noindent\textbf{Multiple-hypothesis testing: Image selection.}
For each image $\X$, let $\hat g(\X)$ denote the caption produced by \texttt{BLIP} and $\Y_{j}$ the human reference. 
We quantify mismatch via
\[
s(\Y, \hat g(\X)) \;=\; 1 \;-\; \operatorname{ROUGE\!-\!L}\bigl(\Y,\hat g(\X)\bigr)\in[0,1],
\]
where ROUGE-L captures the longest common subsequence overlap between the two captions. Small $s(\Y, \hat g(\X))$ values (large ROUGE-L) indicate tight semantic alignment. 
Following \cite{gui2024conformal}, we frame image selection as a simultaneous testing problem: for $j= 1,\cdots,n_{\text{test}}$,
\[
H_{0,j}\!:\; s(\Y_j, \hat g(\X_j)) \,\ge\, c
\quad\text{versus }\quad
H_{a,j}\!:\; s(\Y_j, \hat g(\X_j))\,<\, c,
\]
with decision threshold \(c = 0.7\). 
Rejecting \(H_{0,j}\) labels image $\X_j$ as well-captioned and eligible for downstream use while controlling the overall false-discovery rate (FDR) across all 1,000 images in the inference sample.

\change{The goal is to identify the set of images that \texttt{BLIP} captions well while controlling the false discovery rate of erroneous selections. During generative-model training, \texttt{BLIP} image embeddings serve as covariates for the conditional diffusion model. A synthetic sample of size $m=1000$ is drawn from this model to approximate $P_{s|\x_j}\bigl(s(\Y_j,\hat g(\X_j)) \ge c\bigr)$ for each image, yielding the $p$-values used in the Benjamini--Hochberg step \citep{benjamini1995controlling}.} 

The testing pipeline is designed as follows:


\begin{enumerate}[leftmargin=1.2em]
\item \emph{Conformal $p$-value construction.} 
Using the calibration set, we train a conditional diffusion model that maps \texttt{BLIP} image features to the distribution of BLIP-to-reference dissimilarities. For each test image $j$, the model estimates the probability that $H_{0,j}$ holds, $P_{s|\x_j}\bigl(s(\Y_j,\hat g(\X_j)) \ge c\bigr)$, from which we compute the conformal $p$-value \citep{jin2023selection}, 
\begin{equation*}
p_{j}\;=\;\frac{1+ \sum_{i=1}^{n_{\text{val}}} q_i}{n_{\text{val}}+1},
\end{equation*}
where $n_{\text{val}}=1,000$ and $q_i$ is given by the indicator
{\small\begin{equation*}
I\bigl(y_{i}=0,P_{s|\x_i}(s(Y_i, \hat g(\X_i))\ge c)\ge P_{s|\x_j}(s(Y_j, \hat g(\X_j))\ge c)\bigr),
\end{equation*}}

\item \emph{FDR control.} 
Apply the Benjamini--Hochberg procedure to the vector of $p$-values: reject $H_{0,j}$ if 
\(
p_{j}\le \alpha\,\hat k/n_{\text{test}},
\)
where 
\(
\hat k=\max\{k:p_{(k)}\le\alpha k/n_{\text{test}}\},
\)
and $n_{\text{test}}=1,000$. 
\end{enumerate}

\noindent\textbf{Evaluation.} 
Figure~\ref{fig_cap1} compares realized FDR (left) and power (right) for GSI and the Conformal Alignment baseline across target FDR levels from 0 to 0.5. 
Ideal behavior lies on the diagonal $y=x$. 
GSI adheres closely to the nominal FDR while achieving uniformly higher power than CA, demonstrating more reliable image selection.

As in the hallucination-detection example, GSI models the entire conditional distribution of dissimilarities via diffusion, yielding caption-specific credibility sets that detect subtle semantic errors. CA, by contrast, relies on a single XGBoost point estimate and can miss fine-grained mismatches, leading to lower power at comparable FDR levels.

\begin{figure}[h]
\centering
\includegraphics[width=.95\linewidth]{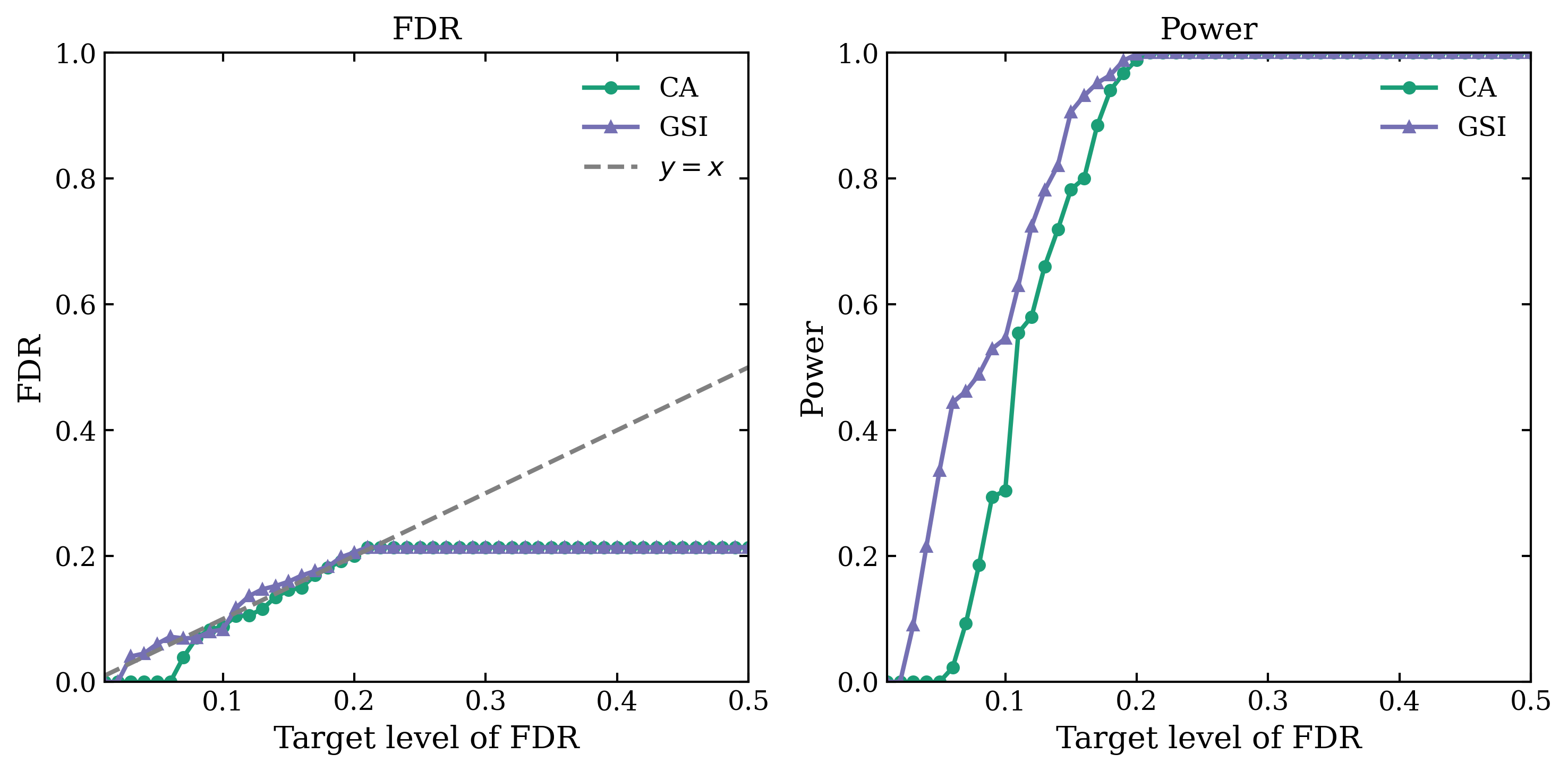}
\caption{FDR and power for CA \citep{gui2024conformal} and GSI across target FDR levels; the line $y=x$ represents perfect FDR control.}
\label{fig_cap1}
\end{figure}

\change{\noindent\textbf{Computational cost.}
GSI with diffusion-based generators typically incurs higher computational cost. Relative to CP and CUQR, the additional cost arises from fitting and sampling the conditional score generator. Appendix~C quantifies this trade-off. In the tabular generator comparison, diffusion training times range from 65 to 214 seconds, and inference time from 69 to 361 seconds, across the five datasets; by contrast, the alternative generators are substantially faster at inference time. In the multimodal tasks, the forward pass of the pretrained LLM or BLIP model is shared across methods, so the incremental cost of GSI again comes from the score-generation stage. For question answering, Appendix~C reports training/inference times of 51/67 seconds for GSI, compared with 44/negligible for CA and 843/14{,}400 for SE; for image captioning, the corresponding training/inference times are 128/333 seconds for GSI and 69/negligible for CA. Thus, diffusion-based GSI is more expensive than CA and lighter conditional generators because it requires both generator training and reverse-process sampling, but it remains substantially cheaper than SE in the Q\&A task, where inference requires collecting multiple candidate answers from the LLM and is therefore computationally intensive.}

\section{Conclusion}
\label{discussion}

This paper introduces Generative Score Inference, a unified uncertainty quantification framework for constructing prediction sets for multimodal tasks using high-fidelity generative models such as diffusion models. This approach offers a statistically guaranteed inference method and has broad applicability across various domains. Compared to a strong competitor in the literature, conformal inference, GSI leverages generative models to better estimate the conditional score distribution rather than relying on empirical distributions. GSI is particularly suitable for complex problems and multimodal distributions. In our experiments, GSI yields prediction intervals that are shorter than those produced by conformal inference while maintaining comparable marginal coverage and achieving better conditional coverage, thereby providing tighter uncertainty quantification.

Our primary goal is to provide researchers with tools that foster reliable conclusions from data. These tools have the potential to enhance the credibility and reliability of data-driven discoveries and strengthen statistical inferences across a broad range of applications. In this spirit, applications to hallucination detection and image-caption selection further demonstrate GSI's practical value. \change{At the same time, the effectiveness of GSI depends on the quality of the learned conditional score generator and on the availability of suitable supervision for defining the score in a given application. These dependencies are especially visible in reference-based tasks such as hallucination detection, where offline calibration data with verified answers remain essential.}

While our analysis has focused on supervised settings with available reference outputs, the underlying principles of GSI extend naturally to unsupervised tasks. In domains such as anomaly detection, the score function can be defined via unsupervised or self-supervised mechanisms, enabling GSI to quantify uncertainty even without labeled data. This highlights a broader opportunity for future work: leveraging generative models not just for prediction, but as a foundation for principled statistical inference in complex, label-sparse environments.

\appendix
\begin{center}
\Large\textbf{Appendix}
\end{center}

\section{Proofs}
\label{sec: proofs}

\begin{proof}[Proof of Theorem 3.2:] Let $\hat P^{(m)}_{s|\x_{\text{new}}}$ denote the empirical distribution based on a synthetic sample of size $m$ given $\x_{\text{new}}$. Note that
\begin{align*}
P_{\y|\x_{\text{new}}}&(\Y_{\text{new}}\in \mathcal{C}_{\alpha}(\x_{\text{new}}))=
P_{s|\x_{\text{new}}}(
s(\Y_{\text{new}}, \hat f(\x_{\text{new}})) \leq q_{1-\alpha,m}
)\\
= & \underbrace{P_{s|\x_{\text{new}}}(
s(\Y_{\text{new}}, \hat f(\x_{\text{new}})) \leq q_{1-\alpha,m}
)-\hat{P}_{s|\x_{\text{new}}}(
s(\Y_{\text{new}}, \hat f(\x_{\text{new}})) \leq q_{1-\alpha,m}
)}_{I_1}\\
&+\underbrace{\hat{P}_{s|\x_{\text{new}}}(
s(\Y_{\text{new}}, \hat f(\x_{\text{new}})) \leq q_{1-\alpha,m})-\hat{P}^{(m)}_{s|\x_{\text{new}}}(s(\Y_{\text{new}}, \hat f(\x_{\text{new}})) \leq q_{1-\alpha,m})}_{I_2}
\\
&+ \hat{P}^{(m)}_{s|\x_{\text{new}}}(
s(\Y_{\text{new}}, \hat f(\x_{\text{new}})) \leq q_{1-\alpha,m}).
\end{align*}

The first term $I_1$ can be bounded by the TV-norm generation error, while the second term $I_2$ is the estimation error based on a sample of size $m$.
\begin{align*}
|I_1|&=|P_{s|\x_{\text{new}}}(
s(\Y_{\text{new}}, \hat f(\x_{\text{new}})) \leq q_{1-\alpha,m}
) -\hat{P}_{s|\x_{\text{new}}}(
s(\Y_{\text{new}}, \hat f(\x_{\text{new}})) \leq q_{1-\alpha,m}
)|\\
&\leq \sup_{z} |P_{s|\x_{\text{new}}}(z)-\hat P_{s|\x_{\text{new}}}(z)|
\leq {\mathrm{TV}(P_{s|\x_{\text{new}}},\hat P_{s|\x_{\text{new}}})}.
\end{align*}
By the Dvoretzky-Kiefer-Wolfowitz (DKW) inequality \citep{dvoretzky1956asymptotic}, with probability at least $1-2\exp(-2m\varepsilon^2)$, we have
$$|I_2|\leq \sup_{z} |\hat{P}_{s|\x_{\text{new}}}(z)-\hat{P}^{(m)}_{s|\x_{\text{new}}}(z)|
\leq \varepsilon.$$
Then, given the event $\{\mathrm{TV}(P_{s|\x_{\text{new}}},\hat{P}_{s|\x_{\text{new}}})\leq \tau\}$, with probability more than $1-2\exp(-2m\varepsilon^2)$,
we have
$$|P_{\y|\x}(\Y_{\text{new}} \in \mathcal{C}_{\alpha}(\x_{\text{new}}))-(1-\alpha)|\leq \varepsilon+{\tau}.
$$
The desired result follows from Assumption 3.1.
\end{proof}

\section{Diffusion models for score generation}
\label{sec: diff}
Recent advancements in generative modeling provide promising avenues to address challenges in prediction through data generation. Generative models like diffusion models \citep{sohl2015deep, ho2020denoising, zhang2023text, lin2023diffusion, yuan2023spatio, kotelnikov2023tabddpm, zheng2022diffusion, kim2022stasy} tailored for different domains can generate synthetic data closely approximating the data-generating distribution.

In this section, we will describe a conditional diffusion model for generating conditional score values in the GSI method and discuss its generation accuracy theory.
\subsection{Diffusion}

\noindent \textbf{Forward process. } The forward process in the diffusion model systematically transforms a random vector $\Z(0)$ into white noise by progressively
injecting white noise into a differential equation defined with the Ornstein-Uhlenbeck process, leading to diffused distributions from the initial state $\Z(0)$:
\begin{equation}
\label{forward}
\mathrm{d}\Z(\tau)=-{b}_{\tau} \Z(\tau)\mathrm{d}\tau+\sqrt{2{b}_\tau}\mathrm{d}W(\tau),\quad \tau \geq 0,
\end{equation}
where $\Z(\tau)$ has a probability density $p_{\z(\tau)}$, $\{W(\tau)\}_{\tau\geq 0}$ represents a standard Wiener process and ${b}_t$ is a non-decreasing weight function. Under \eqref{forward}, $\Z(\tau)$ given $\Z(0)$ follows $N(\mu_{\tau}\Z(0),\sigma^2_{\tau}\bm I)$, where $\mu_{\tau}=\exp(-\int_0^\tau {b}_s\mathrm{d} s)$ and $\sigma^2_{\tau}=1-\mu_{\tau}^2$. Here, setting
${b}_s=1$ results in $\mu_{\tau}=\exp{(-\tau)}$ and $\sigma^2_{\tau}=1-\exp{(-2\tau)}$. Practically, the process terminates at a sufficiently large $\overline{\tau}$, ensuring the distribution of $\Z(\tau)$,
a mixture of $\Z(0)$ and white noise, resembles the standard Gaussian vector.

\noindent \textbf{Backward process.} Given $\Z(\overline{\tau})$ in \eqref{forward}, a backward process is employed for sample generation for $\Z(0)$. Assuming \eqref{forward} satisfies certain conditions \citep{anderson1982reverse}, the backward process $\bm V(\tau)=\Z(\overline{\tau}-\tau)$, starting with $\Z(\overline{\tau})$, is derived as:
\begin{align}
\label{reverse}
\mathrm{d}\bm V(\tau)={b}_{\overline{\tau}-\tau}(\bm V(\tau)+2\nabla\log p_
{\z(\overline{\tau}-\tau)}(\Z(\overline{\tau}-\tau))\mathrm{d}\tau +\sqrt{2{b}_{\overline{\tau}-\tau}}\mathrm{d}W(\tau); \quad \tau \geq 0,
\end{align}
where $\nabla\log p_{\z}$ is the score function which represents the gradient of $\log p_{\z}$.

\noindent \textbf{Score matching.} To estimate the unknown score function, we minimize a matching loss between the score and its approximator $\theta$:
$\int_{0}^{\overline{\tau}}\mathrm{E}_{\z(\tau)}\|\nabla \log p_{\z(\tau)}(\Z(\tau))-\theta(\Z(\tau),\tau)\|^2\mathrm{d}\tau$,
where $\|\x\|=\sqrt{\sum^{d_x}_{j=1}\x_j^2}$ is the Euclidean norm, which is equivalent to minimizing the following loss \citep{oko2023diffusion},
\begin{equation}
\label{loss_2}
\int_{\underline{\tau}}^{\overline{\tau}}\mathrm{E}_{\z(0)}\mathrm{E}_{\z(\tau)|\z(0)}\|\nabla \log p_{\z(\tau)|\z(0)}(\Z(\tau)|\Z(0))-\theta(\z(\tau),\tau)\|^2\mathrm{d}\tau,
\end{equation}
with $\underline{\tau}=0$. In practice, to avoid score explosion due to $\nabla \log p_{\z(\tau)|\z(0)} \rightarrow \infty$ as $\tau\rightarrow 0$, we restrict the integral interval to $\underline{\tau}>0$ \citep{oko2023diffusion,chen2023improved} in the loss function.
Then, both the integral and $\mathrm{E}_{\z(\tau)|\z(0)}$ can be precisely approximated by sampling $\tau$ from a uniform distribution on $[\underline{\tau},\overline{\tau}]$ and a sample of $\Z(0)$ from the conditional distribution of $\Z(\tau)$ given $\Z(0)$.

\noindent \textbf{Generation.} To generate a random sample of $\bm V(\tau)$, we replace the score $\nabla\log p_{\Z(\overline{\tau}-\tau)}$ by its estimate $\hat \theta$ in \eqref{reverse} to yield $\bm V(\tau)$ in the backward equation. For implementation, we may utilize a discrete-time approximation of the sampling process, facilitated by numerical methods for solving stochastic differential equations, such as Euler-Maruyama and stochastic Runge-Kutta methods \citep{song2020denoising}.

\subsection{Conditional score generation \label{sec_3-2}}

To generate the score $s$ given $\X$, we use a conditional diffusion model to learn the conditional probability density $p_{s|\x}$, as described in \eqref{forward}-\eqref{reverse}.

First, we assign $\Z(0)=s$ to our score generation task in \eqref{forward}.
Given a target training sample $(\x_i,s_i)_{i=1}^{n_s}$, we follow \eqref{forward}-\eqref{reverse} to construct an empirical score matching loss $L(\theta_s)=\sum_{i=1}^{n_s} l(\x_i,s_i;\theta_s)$
in \eqref{loss_2} with
\begin{align*}
\label{loss-diffusion}
l_s(\x_i,s_i;\theta)=
\int_{\underline{\tau}}^{\overline{\tau}} \mathrm{E}_{s(\tau)|s(0)}\|\nabla \log p_{s(\tau)|s(0)}(s(\tau)|s_i)-\theta_s(s(\tau),\x_i,\tau)\|^2 \mathrm{d}\tau,
\end{align*}
where $(\underline{\tau}, \overline{\tau})$ denotes early stopping
for $(0,+\infty)$. The estimated score $\hat \theta_s(s(\tau),\x,\tau)=\arg\min_{\theta_s\in\Theta_s}L_s(\theta_s)$.
We will use the neural network for $\Theta_s$.

\noindent \textbf{Neural network. } An $\mathbb{L}$-layer network $\Phi$ is defined by a composite function
$
\Phi(\x)=(\bm{\mathrm{A}}_\mathbb{L}\sigma(\cdot)+\bm{b}_\mathbb{L})\circ\cdots(\bm{\mathrm{A}}_2\sigma(\cdot)+\bm{b}_2)\circ (\bm{\mathrm{A}}_1\x+\bm{b}_1),
$
where $\bm{\mathrm{A}}_i\in \mathbb{R}^{d_{i+1}\times d_i}$ is a weight matrix and $\bm b_i \in \mathbb{R}^{d_{i+1}}$ is the bias of a linear transformation of the $i$-th layer, and $\sigma$ is the ReLU activation function, defined as $\sigma(\x)=\max(\x,0)$.
Then, the parameter space $\Theta$ is set as $\mathrm{NN}(\mathbb{L},\mathbb{W},\mathbb{S},\mathbb{B},\mathbb{E})$
with $\mathbb{L}$ layers, a maximum width of $\mathbb{W}$, effective parameter number $\mathbb{S}$, the sup-norm $\mathbb{B}$, and parameter bound $\mathbb{E}$:
\begin{equation}
\label{p-space}
\begin{aligned}
\mathrm{NN}(&d_{in}, d_{out}, \mathbb{L}, \mathbb{W}, \mathbb{S}, \mathbb{B}, \mathbb{E})
= \\
\Bigl\{\,& \Phi:\;
d_1 = d_{in},\; d_{\mathbb{L}+1} = d_{out},\;
\max_{1\le i\le \mathbb{L}} d_i \le \mathbb{W},\\
&
\sum_{i=1}^{\mathbb{L}}\!\bigl(\|\bm{\mathrm{A}}_i\|_{0}+\|\bm{b}_i\|_{0}\bigr)\le \mathbb{S},\;
\|\Phi\|_{\infty}\le \mathbb{B},\\
&
\max_{1\le i\le \mathbb{L}}\!\bigl(\|\bm{\mathrm{A}}_i\|_{\infty},\|\bm{b}_i\|_{\infty}\bigr)\le \mathbb{E}
\Bigr\}.
\end{aligned}
\end{equation}

where $\|\cdot\|_{\infty}$ is the maximal magnitude of entries and $\|\cdot\|_0$ is the number of nonzero entries.

Now, we approximate \eqref{reverse} by substituting $\nabla \log p_{s(\tau)|\x}$ with $\hat{\theta}_t$, resulting in
\begin{align}
\label{s-reverse}
\mathrm{d}\hat{\V}(\tau)={b}_{\overline{\tau}-\tau}(\hat{\V}(\tau)+2 \hat{\theta}_s(\hat{\V}(\tau),\x,\overline{\tau}-\tau) )\mathrm{d}\tau
+\sqrt{2{b}_{\overline{\tau}-\tau}}\mathrm{d}W(\tau), \tau\in[0,\overline{\tau}-\underline{\tau}_t],
\end{align}
where we use $\hat{\V}(\overline{\tau}-\underline{\tau})$ for sample
generation to replicate $s(0)$, from the initial state $\hat{\V}(0)\sim N(0,1)$ in \eqref{s-reverse}. The resulting density $\hat{p}_{s|\x}$ is represented by $p_{\hat{\bm v}(\overline{\tau}-\underline{\tau})|\x}$.

Next, we introduce the smoothness assumption specific to diffusion models.

\noindent \textbf{Smooth class. } Let $\bm \alpha$ be multi-index with $|\bm \alpha| \leq \lfloor
r\rfloor$, where $\lfloor r\rfloor$ is the integer part of $r>0$. A H\"older ball $\mathcal{H}^{r}(\mathcal{D},\mathbb{R}^m,B)$ of radius $B$ with the degree of smoothness $r$ from domain $\mathcal{D}$ to $\mathbb{R}^m$ is defined by:
{\small
\begin{multline*}
\biggl\{\, g=(g_1,\cdots,g_m):\ \max_{1\le l\le m}\Bigl(
\max_{|\bm\alpha|\le \lfloor r\rfloor}\ \sup_{\x}\, \bigl|\partial^{\bm\alpha} g_l(\x)\bigr|
+\ \max_{|\bm\alpha|=\lfloor r\rfloor}\ \sup_{\x\ne\y}\,
\frac{\bigl|\partial^{\bm\alpha} g_l(\x)-\partial^{\bm\alpha} g_l(\y)\bigr|}
     {\|\x-\y\|^{\,r-\lfloor r\rfloor}}
\Bigr)\ <\ B \biggr\}.
\end{multline*}}

\begin{assumption}
\label{A_d_c}
Assume that $p^0_{s|\x}(s|\x) = \exp(-c_fs^2 / 2) \cdot g(s,\x)$, where $g$ belongs to $\mathcal{H}^{r}(\mathbb{R} \times [0,1]^{d_x},\mathbb{R},B)$ for a constant radius $B>0$ and $c_f>0$ is a constant. Assume that $g$ is lower bounded away from zero with
$f \geq \underline{c}$.
\end{assumption}

With the assumed smoothness, we give the results of the generation accuracy for conditional diffusion models.

\begin{theorem}[Generation error of diffusion models]
\label{thm_diff_general}
Under Assumption \ref{A_d_c}, setting the neural network's structural hyperparameters of $\Theta=\mathrm{NN}(d_x+2,1,\mathbb{L},\mathbb{W},\mathbb{S},\mathbb{B},\mathbb{E})$ as follows: $\mathbb{L}=c_L \log^4K$, $\mathbb{W}= c_W K\log^7K$, $\mathbb{S}= c_S K\log^9K$,
$\log\mathbb{B}= c_B\log K$, $\log\mathbb{E}= c_E\log^4K$, with diffusion stopping criteria from \eqref{forward}-\eqref{reverse}
as $\log \underline{\tau}= -c_{\underline{\tau}}\log K$ and $\overline{\tau}=c_{\overline{\tau}}\log K$, where $\{c_L,c_W, c_S, c_B,c_E,c_{\underline{\tau}},c_{\overline{\tau}}\}$ are sufficiently large constants, yields the error in diffusion generation :
For any small real number $\tau>0$,
\begin{eqnarray}
\label{c-rate-general}
& P(\mathrm{TV}(p^0_{s|\x},\hat{p}_{s|\x})\geq \tau)\leq \beta(\tau,n_s)=
\frac{c_0{n_s}^{-\frac{r}{d_x+1+2r}}\log^k {n_s}}{\tau},
\end{eqnarray}
with some constant $c_0>0$ and $k>0$.
\end{theorem}

\begin{proof} Under Assumption \ref{A_d_c}, it follows from Theorem 2 of \cite{tian2024enhancing}
that
\begin{eqnarray*}
& P(\mathrm{E}_{\x}[\mathrm{TV}(p^0_{s|\x},\hat{p}_{s|\x})]\geq a\delta_{n_s})\leq \exp(-c_e n^{1-\xi} (a \delta_{n_s})^2),
\end{eqnarray*}
for any $a\geq 1$, some constant $c_e>0$ and a small $\xi>0$, with $\delta_{n_s}\asymp{n_s}^{-\frac{r}{d_x+1+2r}}\log^k {n_s}$. Hence,
$\mathrm{E}_{\mathcal{D}}\mathrm{E}_{\x}[\mathrm{TV}(p^0_{s|\x},\hat{p}_{s|\x})]\asymp {n_s}^{-\frac{r}{d_x+1+2r}}\log^k {n_s}$, where
$\mathrm{E}_{\mathcal{D}}$ denotes the expectation with respect to original data
$\mathcal{D}$. By Markov's inequality, the tail probability bound for $p^0_{s|\x},\hat{p}_{s|\x}$ is established:
\begin{eqnarray*}
P(\mathrm{TV}(p^0_{s|\x},\hat{p}_{s|\x})\geq \tau)& \leq & \frac{\mathrm{E}_{\mathcal{D}}\mathrm{E}_{\x}[\mathrm{TV}(p^0_{s|\x},\hat{p}_{s|\x})]}{\tau}\\
& \leq & \frac{c_0{n_s}^{-\frac{r}{d_x+1+2r}}\log^k {n_s}}{\tau}.
\end{eqnarray*}
This completes the proof.
\end{proof}

\section{Experiment details and additional results}
\label{exp-details}

\subsection{Experiment details}
This section outlines the modeling approaches used in the experiments described in Section 4.

In the first example, we adopt a gradient boosting regressor as the prediction model, trained on the same data and with identical hyperparameter settings as those in \cite{alaa2023conformalized} to ensure comparability. Across all five datasets, we apply the diffusion model settings listed in Table~\ref{tab:par_diff}, adapting only the conditioning dimensions to each dataset. The synthetic sample size is set as $m=1000$.

\change{From a computational perspective, the main overhead of GSI relative to CP or CUQR is the training and sampling of the conditional score generator. For tabular experiments, this generator operates on scalar scores and is therefore lightweight relative to the base predictive model; its memory footprint is driven mainly by the covariate dimension, hidden width, and batch size. In the LLM and image-captioning experiments, the pretrained predictor forward pass is common to all methods, so the incremental cost of GSI is the additional diffusion-model fitting on calibration scores, together with Monte Carlo sampling at test time. The architecture and hyperparameter search space reported below make these costs explicit and reproducible.}

The configurations for the second example are as follows:
\begin{itemize}
\item \textbf{Answer generation:} \texttt{LLaMA-3.1-8B-Instruct} model. The Llama-3.1-8B-Instruct model \cite{llama3modelcard} is used to generate answers and extract embeddings of dimension 4096 for the question text conditions.Specifically, we set \texttt{temperature=1.0}, \texttt{top\_p=0.9}, \texttt{top\_k=20}.The prompt to generate the answer is:
\begin{quote}
\ttfamily
Answer the following question in a single, brief but complete sentence.\\
Question: \textit{\{q\}}\\
Answer:
\end{quote}
\item \textbf{Diffusion model:} We use a conditional diffusion model that jointly encodes both time and contextual covariates. The model architecture is defined as follows:

\begin{itemize}
\item \textbf{Input:} A noisy scalar value $s \in \mathbb{R}$, conditional input of the text embedding $x \in \mathbb{R}^d$ with $d=4096$, and time step $t \in [1, T]$.
\item \textbf{Condition Encoding:} A multi-layer perceptron encodes the condition $x$ via
$
h_{\text{cond}} = f_{\text{cond}}(x) = \text{MLP}(x) \in \mathbb{R}^{h},
$
where $h$ is the hidden dimension.
\item \textbf{Time Embedding:} The timestep $t$ is embedded using a learned transformation:
$
h_t = f_{\text{time}}(t) = \text{MLP}(t) \in \mathbb{R}^{e},
$
where $e$ is the time embedding dimension.
\item \textbf{Diffusion Backbone:} The final input is the concatenation of $(x, h_{\text{cond}}, h_t)$, followed by residual blocks:
$
h = \text{ReLU}(W_1[x, h_{\text{cond}}, h_t] + b_1), \quad h = \text{ResBlock}^{(L)}(h).
$
\item \textbf{Output:} The model predicts the noise $\hat{\epsilon}$ using:
$\hat{\epsilon} = W_{\text{out}} h + b_{\text{out}}$.
\end{itemize}

\begin{table}[h]
\caption{Hyperparameter search space for diffusion model training.}
\label{tab:par_diff}
\centering
\scriptsize
\begin{tabular}{l l}
\toprule
\textbf{Hyperparameter} & \textbf{Search Space} \\
\midrule
\texttt{hidden\_dim} & \{64, 128, 256, 512, 1024\} \\
\texttt{time\_embed\_dim} & \{64, 128, 256, 512\} \\
\texttt{layers} & Integer in [3, 15] \\
\texttt{dropout\_p} & Float in [0.0, 0.05] \\
\texttt{learning\_rate} & Log-uniform in $[10^{-5}, 10^{-4}]$ \\
\texttt{noise\_steps} & Integer in [300, 1000] \\
\texttt{Timestep end} & Float in [0.01, 0.02] \\
\texttt{batch size} & \{32, 64, 128\} \\
\bottomrule
\end{tabular}
\end{table}

Each configuration is trained on the training set and evaluated on the validation set using the mean squared error between predicted and true noise vectors. The best-performing model under this objective is selected for downstream calibration and evaluation.

\item \textbf{XGBoost classifier in CA method:} To build a competitive baseline classifier, we employ the XGBoost algorithm with an extensive hyperparameter search using grid-based cross-validation. The following parameters are tuned via 5-fold cross-validation using accuracy as the evaluation metric:

\begin{table}[h]
\caption{Grid search space for XGBoost hyperparameters.}
\label{tab:xgb_param_grid}
\centering
\scriptsize
\begin{tabular}{ll}
\toprule
\textbf{Parameter} & \textbf{Search Space} \\
\midrule
\texttt{learning\_rate} & \{0.1, 0.2\} \\
\texttt{max\_depth} & \{3, 5\} \\
\texttt{min\_child\_weight} & \{3, 5\} \\
\texttt{subsample} & \{0.7, 1.0\} \\
\texttt{colsample\_bytree} & \{0.7, 1.0\} \\
\texttt{n\_estimators} & \{100, 200\} \\
\bottomrule
\end{tabular}
\end{table}

The XGBoost model is initialized with the objective set to \texttt{binary:logistic}, making it suitable for binary classification tasks. The evaluation metric during training is \texttt{mlogloss}.
The model with the best cross-validated performance is selected and used for the testing task.

\end{itemize}

In the third example, we employ the BLIPâ€“imageâ€“captioningâ€“base model to generate captions for images. We set \texttt{max\_length}=40 and \texttt{num\_beams}=10, while retaining default values for all other parameters. The diffusion model in our GSI framework uses the same configuration as in the first example, with the sole exception that its conditioning input is the 768 dimensional image embedding produced by the BLIP model. The synthetic sample size is set as $m=1000$.

\subsection{\change{More experiment results}}
\label{app:more-results}

We conducted an additional generator comparison on all five tabular datasets (\texttt{meps\_20}, \texttt{bio}, \texttt{blog\_data}, \texttt{kin8nm}, and \texttt{naval}) by comparing the diffusion-based GSI results in Table~1 of the main text with three conditional alternatives obtained by replacing the score generator with a conditional variational autoencoder (CVAE,\citet{sohn2015learning}), a conditional generative adversarial network (CGAN, \citet{mirza2014conditional}), or a conditional normalizing-flow-type model based on an affine conditional transformation (CFlow, \citet{winkler2019learning}). For a fair comparison, these alternative generators were implemented with similar multilayer-perceptron building blocks and tuned under a comparable architecture budget.

More specifically, all four generators are trained on the scalar nonconformity scores using the tabular covariates as conditions. The CVAE baseline uses an encoder that maps $(x,s)$ to a Gaussian latent variable and a decoder that maps $(x,z)$ back to the score. The CGAN baseline uses a conditional generator $G(x,z)$ together with a conditional discriminator $D(x,s)$, and training combines an adversarial loss with an $\ell_1$ reconstruction term. The conditional-flow baseline is a conditional affine flow: it learns $\mu(x)$ and $\log \sigma(x)$ and assumes
\[
s \mid x = \mu(x) + \sigma(x)\varepsilon, \qquad \varepsilon \sim N(0,1),
\]
so its likelihood is that of a condition-dependent Gaussian after an affine transformation.

Hyperparameters are tuned separately for each method on the same training/validation split. For the shared architecture and optimization parameters, \texttt{hidden\_dim}, \texttt{layers}, \texttt{dropout\_p}, \texttt{learning\_rate}, and \texttt{batch\_size} use the same search spaces as those in Table~\ref{tab:par_diff}. The remaining method-specific tuning parameters are reported in Table~\ref{tab:generator_tuning_details}. Hyperparameter selection is based on the validation loss corresponding to each model family: denoising mean-squared error for diffusion, reconstruction plus KL loss for the CVAE, adversarial plus reconstruction loss for the CGAN, and negative log-likelihood for the CFlow.

\begin{table}[h]
\centering
\caption{Generator-specific tuning parameters. The shared parameters \texttt{hidden\_dim}, \texttt{layers}, \texttt{dropout\_p}, \texttt{learning\_rate}, and \texttt{batch\_size} use the same search spaces as those in Table~\ref{tab:par_diff}.}
\label{tab:generator_tuning_details}
\scriptsize
\begin{tabular}{p{3.0cm}p{8.5cm}}
\toprule
Method & Additional tuning parameters \\
\midrule
CVAE & \texttt{latent\_dim} $\in \{4,8,16\}$, \texttt{beta\_kl} $\in \{0.001,0.01,0.05\}$. \\
CGAN & \texttt{latent\_dim} $\in \{4,8,16\}$, \texttt{recon\_weight} $\in \{5,10,20\}$. \\
CFlow & No additional method-specific parameter beyond the shared  parameters. \\
\bottomrule
\end{tabular}
\end{table}

\begin{table}[htbp]
\centering
\caption{Additional generator comparison across all five tabular datasets. }
\label{tab:generator_compare_full}
\scriptsize
{%
\begin{tabular}{llccccc}
\toprule
Dataset & Method & $C_{\text{marg}}$ &  $C_{G}$ & $L_{\text{avg}}$ & Train time (s) & Inference time (s) \\
\midrule
\multirow{4}{*}{Meps\_20}
& Conditional Diffusion & 0.91 & 0.80 & 0.98 & 214.66 & 230.03 \\
& Conditional VAE       & 0.88 & 0.60 & 0.83 & 20.71  & 0.05   \\
& Conditional GAN       & 0.77 & 0.00 & 0.55 & 42.06  & 0.06   \\
& Conditional flow      & 0.91 & 0.40 & 0.96 & 17.78  & 0.11   \\
\midrule
\multirow{4}{*}{Bio}
& Conditional Diffusion & 0.89 & 0.88 & 2.19 & 203.29 & 106.51 \\
& Conditional VAE       & 0.90 & 0.88 & 2.38 & 140.92 & 0.16   \\
& Conditional GAN       & 0.75 & 0.70 & 1.63 & 70.83  & 0.07   \\
& Conditional flow      & 0.91 & 0.89 & 2.10 & 53.15  & 0.13   \\
\midrule
\multirow{4}{*}{Kin8nm}
& Conditional Diffusion & 0.90 & 0.85 & 1.82 & 65.40  & 70.82  \\
& Conditional VAE       & 0.90 & 0.83 & 2.17 & 23.67  & 0.02   \\
& Conditional GAN       & 0.59 & 0.50 & 1.20 & 19.67  & 0.03   \\
& Conditional flow      & 0.90 & 0.84 & 1.99 & 8.62   & 0.02   \\
\midrule
\multirow{4}{*}{Naval}
& Conditional Diffusion & 0.89 & 0.85 & 0.96 & 113.33 & 69.17  \\
& Conditional VAE       & 0.90 & 0.83 & 1.21 & 17.20  & 0.03   \\
& Conditional GAN       & 0.64 & 0.53 & 0.59 & 8.78   & 0.03   \\
& Conditional flow      & 0.91 & 0.86 & 0.88 & 38.37  & 0.04   \\
\midrule
\multirow{4}{*}{Blog}
& Conditional Diffusion & 0.90 & 0.80 & 1.66 & 194.03 & 361.46 \\
& Conditional VAE       & 0.91 & 0.60 & 1.80 & 141.84 & 0.17   \\
& Conditional GAN       & 0.65 & 0.47 & 0.72 & 26.73  & 0.10   \\
& Conditional flow      & 0.90 & 0.80 & 1.61 & 19.60  & 0.28   \\
\bottomrule
\end{tabular}%
}
\end{table}

\begin{table}[htbp]
\centering
\caption{Training and inference time comparison across tabular, Question\&Answer, and Image Captioning tasks. For the tabular tasks, we report the diffusion-based GSI timing together with the mean training and test time of the new TCP and CP runs. Entries smaller than 0.01 seconds are replaced with ``--'' and treated as negligible.}
\label{tab:time_compare}
\scriptsize
{%
\begin{tabular}{lllcc}
\toprule
Task & Dataset & Method & Train time (s) & Inference time (s) \\
\midrule
\multirow{15}{*}{Tabular Prediction}
& \multirow{3}{*}{Meps\_20} & GSI  & 214.66 & 230.03 \\
&                         & CP   & --     & --     \\
&                         & CUQR & 0.02   & 4.56   \\
\cmidrule(lr){2-5}
& \multirow{3}{*}{Bio}    & GSI  & 203.29 & 106.51 \\
&                         & CP   & --     & --     \\
&                         & CUQR & 0.03   & 4.71   \\
\cmidrule(lr){2-5}
& \multirow{3}{*}{Kin8nm} & GSI  & 65.40  & 70.82  \\
&                         & CP   & --     & --     \\
&                         & CUQR & 0.02   & 0.81   \\
\cmidrule(lr){2-5}
& \multirow{3}{*}{Naval}  & GSI  & 113.33 & 69.17  \\
&                         & CP   & --     & --     \\
&                         & CUQR & 0.01   & 1.28   \\
\cmidrule(lr){2-5}
& \multirow{3}{*}{Blog} & GSI  & 194.03 & 361.46 \\
&                             & CP   & --     & 0.01   \\
&                             & CUQR & 0.01   & 48.03  \\
\midrule
\multirow{3}{*}{Question\&Answer}
& \multirow{3}{*}{WikiQA} & GSI & 51.24  & 67.49    \\
&                         & CA  & 44.76  & --    \\
&                         & SE  & 843.52 & 14400.18 \\
\midrule
\multirow{2}{*}{Image Captioning}
& \multirow{2}{*}{COCO}   & GSI & 128.17 & 333.21 \\
&                         & CA  & 69.81  & --  \\
\bottomrule
\end{tabular}%
}
\end{table}

Table~\ref{tab:generator_compare_full} reports the resulting interval quality and computational summaries for every dataset-method pair.
Across the five datasets, the diffusion-based generator and the conditional flow alternative achieve similar marginal coverage overall, but diffusion yields stronger subgroup-coverage performance and remains more stable under the joint criteria of marginal coverage, subgroup robustness, and interval efficiency. The CVAE is competitive on several datasets but typically produces longer intervals, whereas the CGAN consistently undercovers despite shorter intervals. Computationally, diffusion is more expensive than the alternative generators, especially in the multimodal tasks and in the test-time Monte Carlo sampling stage, while the alternative generators are substantially cheaper. These results support the use of diffusion as a robust default generator while clarifying the tradeoff against lighter conditional generators.

Table~\ref{tab:time_compare} summarizes the training and inference time for the tabular, Q\&A, and Image Captioning tasks. For the tabular tasks, we report the timing of diffusion-based GSI from Table~\ref{tab:generator_compare_full}, together with the training and test time of the CP and CUQR runs. The computational cost of CP is essentially negligible, since it only requires a simple quantile computation on calibration scores. The test-time cost of CUQR is somewhat higher, but it remains below that of diffusion-based GSI. Although GSI, especially with a diffusion backbone, incurs higher training and inference costs, it offers greater flexibility once the conditional score generator has been trained: in our implementation, after synthetic scores are generated, uncertainty levels can be adjusted without repeatedly fitting new quantile-specific models.

For the multimodal tasks, relative to the CA method, whose inference cost is negligible because it only requires a forward pass through a discriminator model, diffusion-based GSI has a higher inference cost because sample generation requires running a reverse-time denoising procedure. Nevertheless, its inference cost remains much smaller than that of the SE method, which must generate a collection of responses for each query at test time. We further note that the computational cost of diffusion-based GSI can be reduced by using faster samplers, such as DPM-Solver methods \citep{lu2022dpm,lu2025dpm}, which are also of independent interest in the broader study of accelerated diffusion-based generation.

\end{document}